\documentclass[sts]{imsart}

\usepackage[utf8]{inputenc} 
\usepackage[T1]{fontenc}    
\usepackage{hyperref}       
\usepackage{url}            
\usepackage{booktabs}       
\usepackage{amsfonts}       
\usepackage{nicefrac}       
\usepackage{subfig}
  \usepackage{nicematrix}
\usepackage{microtype}      
\usepackage{xcolor}         
\usepackage{amsmath,amsthm,amssymb}
\newtheorem{theorem}{Theorem}[section]

\newtheorem{remark}[theorem]{Remark}

\usepackage{enumitem}
  \usepackage{tikz}
    \usetikzlibrary{arrows,shadows,positioning}
\def\circledarrow#1#2#3{ 
\draw[#1,->] (#2) +(80:#3) arc(80:-260:#3);
}

\usepackage[linesnumbered,ruled,vlined]{algorithm2e}

\SetKwInput{KwInput}{Input}                
\SetKwInput{KwOutput}{Output}              
\SetKwInput{KwData}{Hyperparameter}              

\usepackage{xspace}

\newcommand{\SM}{\textsc{SoftModes}}
\renewcommand{\c}{\textbf{c}}
\newcommand{\e}{\textbf{e}}
\newcommand{\BBM}{\textsc{BBM}\xspace}

\begin{document}

\begin{frontmatter}
\title{Clustering Categorical Data: Soft Rounding $k$-modes}
\runtitle{Soft Rounding $k$-modes}
\author{\fnms{Surya Teja} \snm{Gavva}\thanksref{tt}},
\author{\fnms{} \snm{Karthik C.\ S.}\thanksref{t}}
\and
\author{\fnms{Sharath} \snm{Punna}\thanksref{ttt}}

\thankstext{tt}{Email: suryateja@math.rutgers.edu
}
\thankstext{t}{Email: karthik.cs@rutgers.edu}
\thankstext{ttt}{Email: sharath.punna@rutgers.edu}
\runauthor{S.\ Gavva, Karthik C.\ S., and S.\ Punna}

\affiliation{Rutgers University}

\begin{abstract}
Over the last three decades, researchers have intensively explored various clustering tools for categorical data analysis. Despite the proposal of various clustering algorithms,  the classical $k$-modes  algorithm remains a popular choice for unsupervised learning of categorical data. Surprisingly, our first insight is that in a natural generative block model, the $k$-modes algorithm performs poorly for a large range of parameters. We remedy this issue by proposing a soft rounding variant of the $k$-modes algorithm (\SM) and theoretically prove that our variant addresses the drawbacks of the $k$-modes algorithm in the generative model.  Finally, we empirically verify that \SM\ performs well on both synthetic and real-world datasets. 
\end{abstract}
\end{frontmatter}

\section{Introduction}

Algorithms to partition datasets into clusters of similar elements are
at the heart of modern machine learning,
and used in a large spectrum of domains ranging from biology to
econometrics, network sciences, physics, chemistry,
etc.. Clustering algorithms are \emph{unsupervised} machine learning
learning techniques: given very
little information about a dataset (namely, no \emph{a priori}
information such as training examples),
they allow to extract hidden patterns in datasets by identifying groups of
points that have higher pairwise similarity than with the rest of the
dataset. Thus, designing efficient
clustering methods has had major impact on a large number of communities
over the last few decades.

Even though the task of minimizing clustering objectives (such as $k$-means or $k$-median) has been shown to be NP-Hard, even to approximate \cite{CKL22}, there are many heuristics, particularly the $k$-means algorithm \cite{Macqueen67} (i.e., the Lloyd's algorithm \cite{Lloyd82}) which seems to perform very well on most real-world datasets. It is for this reason that a 2002 survey of data mining techniques states that $k$-means algorithm ``is by far the most popular clustering
algorithm used in scientific and industrial applications'' \cite{Berkhin2006} and a different paper listed it as one of the top 10 algorithms in data mining \cite{WuKQGYMMNLYZSHS08}. 

However,  if the data contains categorical variables, (Euclidean) geometric approaches are inappropriate and other strategies must be developed \cite{B89}. With that in mind, many algorithms have been developed for clustering categorical data. In 1998, Huang \cite{Huang1997} proposed the $k$-modes algorithm for categorical data clustering. 
This algorithm begins with $k$ arbitrary ``centers,''
typically chosen uniformly at random from the data points (or from a more careful, computation intensive, distance sampling seeding). Each point is then assigned to its
nearest center in Hamming metric, and each center is recomputed as the center of mass of all points assigned to it (i.e., by taking the plurality on each coordinate of the points assigned to it; breaking ties uniformly at random).
These last two steps are repeated until the process stabilizes. One can check that the $k$-median objective (i.e., sum of distances of each point to its closest center) is monotonically
decreasing, which ensures that no configuration is repeated during the course of the algorithm, and thus the procedure always terminates. 

The $k$-modes
algorithm phrased as above, can then be seen as a variation to the $k$-means algorithm, where we calculate the distance between the object and the cluster center by the Euclidean distance
instead of the Hamming distance. After \cite{Huang1997}, many algorithms have been proposed for the clustering of categorical data,  such as ROCK \cite{GuhaRS99},
CACTUS \cite{GGR99}, COOLCAT \cite{BarbaraLC02}, LIMBO \cite{AndritsosTMS04}, wk-modes \cite{CaoLLZ13}, MOGA \cite{MukhopadhyayMB09}, NSGA-FMC \cite{YangKCQ15}, SBC \cite{QianLLLD16},
MOFC \cite{ZhuX18}, and so on.  Despite a plethora of algorithms proposed in literature for clustering categorical data, the (vanilla) $k$-modes algorithm remains a widely popular tool for categorical data analysis. The reasons for this popularity are essentially the same reasons as for $k$-means algorithm (Lloyds algorithm): scalability and simplicity.

Our first contribution is observing that the $k$-modes algorithm does not perform well in a fairly natural generative block model, which we call as the \emph{Boolean Block Model} (\BBM). In the \BBM model, we are given the parameters the number of samples (denoted by $n$), the number of features (denoted by $d$), a partition of the samples into disjoint subsets $C_1,\ldots ,C_k$ (called clusters), a partition of the features into disjoint subsets $D_1,\ldots ,D_k$, and a 	 $k  \times k$ matrix $P$ of inter-cluster and intra-cluster probabilities. 
The data points are then sampled as follows: for every point $a$ in cluster $C_i$, its entry on a feature in $D_j$ is 1 with probability equal to $P(i,j)$ and 0 otherwise.

We remark that if $n=d$, $P$ is symmetric, and for all $i\in [k]$, if $|C_i|=|D_i|$, then we have   that the \BBM model is closely related to the stochastic block model \cite{HOLLAND1983109} which is well-studied in the context of community detection.  Also, the setting when $n\neq d$, arises fairly naturally, for example,   when the dataset  is obtained by responses to questionnaire with Yes/No answers (specifically answers of a population of size $n$ to $d$ many questions), and we would like to identify the two planted clusters. If many of the questions are asking about group memberships, then any particular applicant would typically answer `No' to many of the questions, and this way of encoding leads to sparse weight samples.

To illustrate, our results, suppose we sample a block diagonal matrix $\textbf{A}=\left[\begin{matrix}
\mathbf{B_1} &\mathbf{0}\\
\mathbf{0}&\mathbf{B_2}
\end{matrix}\right]$, where $\mathbf{B_1},\mathbf{B_2}$ are random $d\times n$ matrices where each entry is 1 with probability $p$ and 0 otherwise.  We are interested in the case when $p<0.5$. 
We think of each column in $\mathbf{A}$ as a sample point in $\{0,1\}^{2d}$ and we would like to cluster them into two groups (ideally we would want one group to be the columns of $\mathbf{B}_1$ and the other to be the columns of $\mathbf{B}_2$).   

Our insight is that regardless of the initial seeding of the centers (i.e., both the seeds can be from the same block or from different blocks), the $k$-modes algorithms updates the new centers both to be $\vec{0}$! After this, the new grouping is random, and the centers regardless remain $\vec{0}$.

\begin{theorem}[Informal statement of Theorem~\ref{thm:main}]
In the \BBM model with $d=(\log n)^{\Omega(1)}$, $k=2$, $|C_1|=|C_2|$, $|D_1|=|D_2|$, and $P=\left[\begin{matrix}
p &q\\
q&p
\end{matrix}\right]$, then if $q<p<0.5$ then the accuracy of $k$-modes algorithm is about 0.5.
\end{theorem} 

To further motivate the above theorem's assertion, consider the following scenario: we have access to all the accepted papers to ABC 2022, the top machine learning (ML) conference, and XYZ 2022, the top theoretical computer science (TCS) conference and from this we generate a pool of keywords (features) and for each of the papers in the two conferences, we have an indicator bit (0/1) of whether each of the keywords appeared in that paper. It is reasonable to expect that keywords related to ML (such as ``deep learning'' or even ``artificial intelligence'') tend to appear more in papers of ABC, whereas, keywords related to TCS (such as "algorithms" or "complexity") would appear relatively more in accepted papers of XYZ.

That said, (i) TCS keywords do appear in accepted papers of ABC and ML keywords do appear in papers of XYZ, albeit quantitatively less. (ii) A majority of the ML keywords do not appear in any single ABC paper, as papers tend to be more focussed on the specific problem they address (same holds for TCS keywords in XYZ papers). This justifies (i) noise in non-blocks (i.e., $q<p$) and (ii) sparsity assumption within blocks (i.e., $p<0.5$).
In general, one may consider similar situations arising in textual data mining, where there are a pool of keywords (features) and documents (data points), and we have an identifier bit if the keyword appears in the document.

One may attribute the behavior of the $k$-modes algorithm highlighted in  the theorem above to the hard thresholding features of the majority function, which the algorithm uses to update the centers. A natural approach to remedy this situation is to use soft rounding of the majority function. For example, one may consider replacing the majority function with a sigmoid function that approximates the majority function. However, this only address binary data, and in order to accommodate all kinds of categorical data (i.e., to remedy the hard thresholding of the plurality function), we  propose a general clustering algorithm \SM$(t)$, where $t$ is a hyperparameter. If one sets $t=\infty$ then we obtain the $k$-modes algorithm, and for smaller finite $t$, we obtain a soft rounding version of the $k$-modes algorithm. If fact, \SM$(1)$, for binary data is simply the algorithm which updates the $i^{\text{th}}$ coordinate of the center to 1 with probability $\alpha$, where $\alpha$ is the fraction of 1s in the  $i^{\text{th}}$ coordinate of the points associated to that center. As $t$ increases this rounding becomes more sharp. Note that while soft rounding is a standard technique in machine learning, using the $t^{th}$ power as rounding functions is quite uncommon, and also is a clean theoretical solution as explained in the next section.

   We  show that one can opt for any finite setting of the hyperparameter   $t$ to obtain better classification:

\begin{theorem}[Informal statement of Theorem~\ref{thm:main}]
For every finite $t\ge 1$, in the \BBM model with $d=(\log n)^{\Omega(1)}$, $k=2$, $|C_1|=|C_2|$, $|D_1|=|D_2|$, and $P=\left[\begin{matrix}
p &q\\
q&p
\end{matrix}\right]$, then if $q<p<0.5$ then the accuracy of \SM$(t)$ is about 0.75 with high probability. If $t=1$, and $q$ is sufficiently small, then  \SM$(1)$ \emph{completely} classifies with high probability.
\end{theorem} 

 If the initial seeding of \SM$(t)$  samples two centers from different blocks, then the centers are updated as follows: the center from the first (resp.\ second) block is updated to some point of Hamming weight roughly $np$ in the first (resp.\ second) half of the coordinates and roughly $nq$ in the second (resp.\ first) half of the coordinates. This implies the clusters have stabilized and the algorithm terminates  (see proof of Theorem~\ref{thm:main} for details). 

However, it is surprising that even when the  initial seeding samples of \SM$(t)$ are two centers from the same block, \SM$(t)$ is able to perfectly classify the points (when $q$ is small). The insight that we use here is that even though the two centers are randomly picked from the same block, anti-concentration of the distribution allows us to conclude that both the centers are not of the exact same Hamming weight. Thus one of the centers will be close (relative to the other center) to \emph{every} point in the other block. Therefore after a couple of iterations we will be able to recover the planted clusters. 
We observed empirically that the above result holds even for dense non-zero anti-diagonal matrices although we could only prove the result theoretically for sufficiently sparse anti-diagonal matrices.

We also introduced a new generative model called \emph{Corrupted Codewords Model}, which brings together ideas from Coding theory and statistics (specifically mixture of Gaussians), and empirically verified that \SM\ performs better than both $k$-modes and $k$-means algorithms. 

Additionally, we performed experiments on real-world datasets and noticed that in all experiments \SM\ has higher accuracy than $k$-modes for suitable tuning of the hyperparameter. One possible explanation for this is that  while \SM\ has tendencies to converge to a local minima, the probabilistic nature of the algorithm makes it non-committal to a specific path to converge. In other words, \SM\ escapes bad local minima with some probability. 

Finally, one of the drawbacks of the $k$-modes algorithm over $k$-means algorithm is the issue with rate of convergence. Since the center of a cluster is unique in the Euclidean metric, and the center of a cluster can be far from unique in the Hamming metric, $k$-means algorithm converges very quickly (this is called the ``zero probability condition'' \cite{GG92}). One might wonder that with the soft rounding, and less committal choices, \SM\ might have an even slower rate of convergence than $k$-modes algorithm. This is indeed the case, but we have empirically observed that while \SM\ takes more iterations to converge, it is at every iteration having a higher accuracy than the $k$-modes algorithm and thus for any cap on the number of iterations and \SM\ would  give better results than $k$-modes algorithm.

\subsection{Related Works}
There have been a plethora of clustering algorithms developed for categorical data analysis over the last three decades and some of these were mentioned earlier in the section.  
We mention here a couple of lines of related works that are similar in spirit to our work. 

One may see $k$-distributions \cite{CWJ07}, $k$-histograms \cite{HXDD05}, and $k$-entropies \cite{HPKLLF14} as principally performing a similar center update step as us: using a distribution stemming from the points assigned to the cluster in order to sample a center. In $k$-distributions, each cluster is captured through the joint probability distribution of the alphabet set and these are updated over iterations but such computations are expensive (and in general NP-hard). In $k$-histograms, each cluster is captured by a histogram of the alphabet set and these are updated over iterations but these either do not acknowledge that distinctness of the features or will require large encoding. In $k$-entropies, each cluster is modeled by its probability mass function and updates it, but this algorithm is particularly slow as it follows the Hartigan's method \cite{HW79} over Llyods's method.

Apart from these works, there are other works which consider explicitly measures different from the Hamming measure while performing categorical clustering \cite{goodall1966new,boriah2008similarity,S68,anderberg2014cluster}. Finally, Fuzzy clustering\cite{huang1999fuzzy} might be seen as a different variant where the randomness is over the step of assigning the points to centers instead of the step of updating the center which is considered in this paper. 
 
Finally, we note that clustering tasks in stochastic block models  are typically  by biclustering methods. In particular, the biclustering of discrete data is explicitly studied in the framework of matrix factorization \cite{Miettinen020,HessPHC21}.

\section{Soft Rounding $k$-modes Algorithm (\SM)}\label{sec:soft}

In this section, we describe a generalization (with hyperparameterization) of the   $k$-modes algorithm. In order to do so, we need to introduce the notion of a rounding function.

\paragraph{Rounding functions.} Let $s\in\mathbb{N}$. Let $\Delta_s$ denote the $s$-simplex, i.e., $\Delta_s=\{(x_1,\ldots ,x_{s+1})\in [0,1]^{s+1}:\sum_{i\in[s+1]}x_i=1\}$.  Let $\e_1,\ldots ,\e_{s+1}$ be the extremal points of the simplex (i.e., $\e_i$ is the point which is 1 on the $i^{\text{th}}$ coordinate and 0 everywhere else, and $\c_s:=(\nicefrac{1}{s+1},\ldots ,\nicefrac{1}{s+1})$ be the center of the simplex. 
We say that  $\rho:\Delta_s\to\Delta_s$ is a rounding function if for every $\textbf{x}:=(x_1,\ldots ,x_{s+1})\in \Delta_s$ and $\forall i,i'\in[s+1]$, we have that if $x_i\ge x_{i'}$ then $\rho(\textbf{x})_i\ge \rho(\textbf{x})_{i'}$. Note that we have $\rho(\c_s)=\c_s$ for all rounding functions. We say a family of functions $\mathcal{F}:=\{\rho_s:\Delta_s\to\Delta_s\}_{s\in\mathbb{N}}$ is a family of rounding functions if every $\rho_s$ is a rounding function.

Next, we consider three different families of rounding functions. 

\paragraph{Plurality rounding.} In $\Delta_s$ we identify $2^{s+1}-1$ many centers, which are the centers of every subface of the simplex. The plurality rounding function simply maps every point in the simplex to one of its centers. Intuitively, this rounding function captures the process of outputting the plurality of a function breaking ties uniformly at random. Formally, it is defined as follows. Let
$S_{\max}(\textbf{x})=\{i\in[s+1]:x_i=\|x\|_\infty\}$. Then, the plurality rounding function is given by
$$\displaystyle \rho_s^{\text{plu}}(\textbf{x})_i=\begin{cases} \nicefrac{1}{|S_{\max}(\textbf{x})|}&\text{ if }i\in S_{\max}(\textbf{x})\\ 0&\text{ otherwise}\end{cases}.$$
As an illustration, we have in Figure~\ref{fig:plural}, the rounding function over the triangle. Notice that the only fixed points in the functions are the 7 centers, and every other point is moving towards one of the centers. 

\paragraph{Uniform rounding.} The uniform rounding function  is simply the identity map $\rho_s^{\text{uni}}(\textbf{x})=\textbf{x}$. Intuitively, this rounding function captures the process of outputting the value of an attribute with probability equal to the extent of its appearance in the data/cluster. 
As an illustration, we have in Figure~\ref{fig:uni}, the rounding function over the triangle. Notice that every point is a fixed point.

\paragraph{Soft rounding.}
Soft rounding is designed to have a spectrum of rounding functions whose extreme end points are the plurality rounding and the uniform rounding. With this in mind, we define the following:
 $$\rho_{s,t}^{\text{soft}}(\textbf{x})_i=\frac{x_i^t}{\sum_{i'\in[s+1]}x_{i'}^t}.$$
It is clear that $\rho_{s,\infty}^{\text{soft}}=\rho_{s}^{\text{plu}}$ and $\rho_{s,1}^{\text{soft}}=\rho_{s}^{\text{uni}}$. As an illustration, we have in Figure~\ref{fig:soft}, the rounding function over the triangle for some $t>1$. Notice that the only fixed points are the centers of the simplex but unlike in the plurality rounding function, the soft rounding functions are smooth and slowly displace towards the centers.  Also, notice that both rounding functions in  Figure~\ref{fig:basic} stabilize (i.e., reach a fixed point) after a single iteration, whereas the rounding in Figure~\ref{fig:soft} may take several iterations to stabilize. 

\begin{figure*}
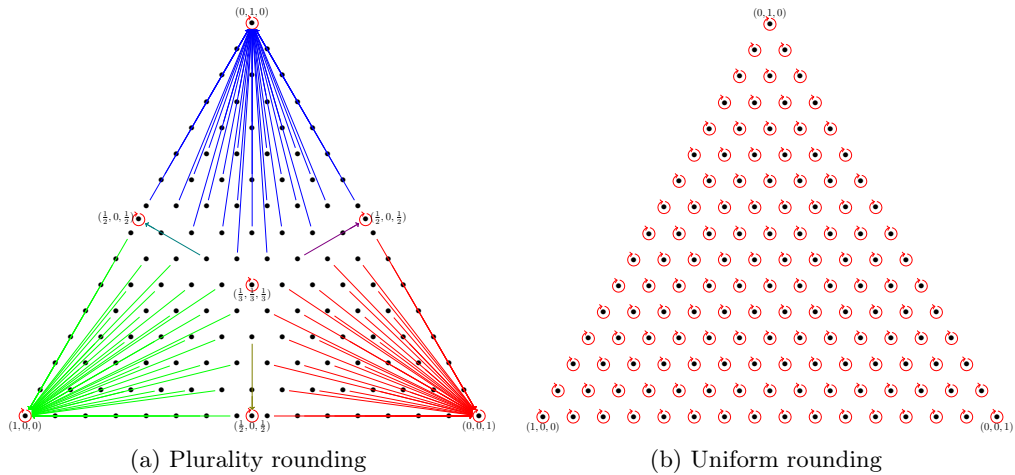

    \centering
    \setkeys{Gin}{width=0.49\textwidth}
\subfloat[Plurality rounding
          \label{fig:plural}]{\resizebox{0.49\textwidth}{!}{\input{pic2}}}
    \hfill
\subfloat[Uniform rounding
          \label{fig:uni}]{\resizebox{0.49\textwidth}{!}{\input{pic1}}}
\caption{Illustrations of rounding functions on the 2-simplex projected on to the plane}
\label{fig:basic}
    \end{figure*}

    \begin{figure*}
    \centering
    \setkeys{Gin}{width=0.7\textwidth}
 {\resizebox{0.7\textwidth}{!}{\input{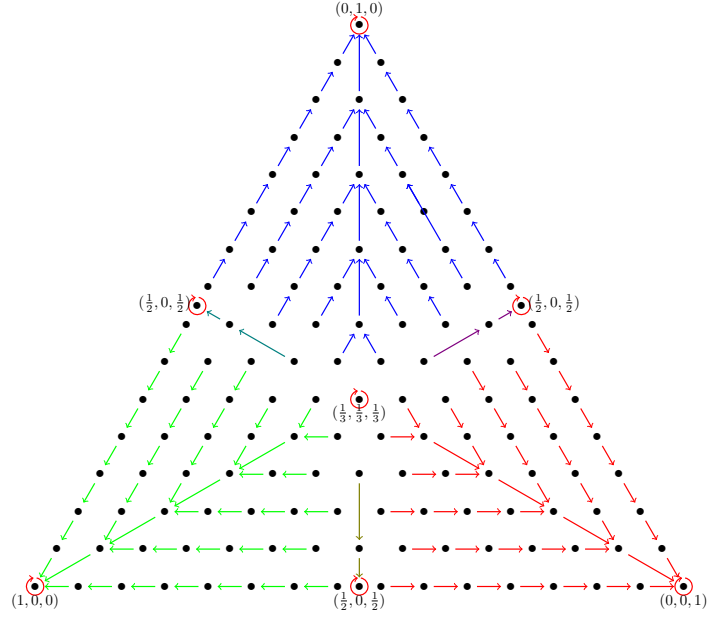}}}
    \caption{Illustration of a soft rounding function on the 2-simplex projected on to the plane}
  
\label{fig:soft}
    \end{figure*}
    
We are now ready to describe our generalization of the $k$-modes algorithm, \SM.
In Algorithm~\ref{alg}, if we use  $\mathcal{F}^{\text{plu}}:=\{\rho_s^{\text{plu}}\}_{s\in\mathbb{N}}$ as the hyperparameter then we simply obtain the $k$-modes algorithm. Throughout the paper, we use the shorthand \SM($t$) to refer to Algorithm~\ref{alg} with hyperparameter $\mathcal{F}:=\{\rho_{s,t}^{\text{soft}}\}_{s\in\mathbb{N}}$. Note that in Algorithm~\ref{alg} we have not specified the termination condition, but for our theoretical results we simply terminate the algorithm if the clusters are not updated and for the experiments we additionally have an upper limit on the number of iterations. Finally,   we denote by $d(\cdot,\cdot)$   the Hamming distance, i.e., the number of coordinates on which the two input points differ. 
    
\begin{algorithm}[!ht]
\DontPrintSemicolon
  
  \KwInput{Categorical dataset $P\subseteq \Sigma_1\times \cdots \times\Sigma_d$ with $d$ features/attributes, where $\forall j\in[d],\  \Sigma_j:=\{1,\ldots ,\sigma_j\}$ (i.e., $\Sigma_j$ is the set of values for the $j^{\text{th}}$ attribute),  $k\in\mathbb{N}$, seeds $c_1^0,\ldots ,c_k^0\in \Sigma_1\times \cdots \times\Sigma_d$}
  \KwOutput{Partition of $P:=P_1\dot\cup\cdots \dot\cup P_k$}
  \KwData{Rounding Function Family $\mathcal{F}:=\{\rho_s:\Delta_s\to\Delta_s\}_{s\in\mathbb{N}}$}
 $r \gets 1$
	
   \While{Termination condition is not satisfied}
   {
   $P_1^r,\ldots ,P_k^r\gets \emptyset$
   
\For{each $p\in P$}{
$j \gets \text{argmin}_{j\in[k]}d(p,c_j^{r-1})$
(breaking ties uniformly at random)

$P_j^r \gets P_j^r \cup \{p\}$
}

\For{each $i\in [k]$}{
\For{each $j\in [d]$}{
Compute $\textbf{x}\in \Delta_{\sigma_j-1}$ where $\forall t\in[\sigma_j], \textbf{x}(t):=\frac{|\{p\in P_i^r: p(j)=t\}|}{|P_i^r|}$

Sample $z\in \Sigma_j$ using the distribution given by $\rho_s(\textbf{x})$

$c_i^r(j)\gets z$
}

}

   	$r \gets r+1$	
   }
   Output $P_1^{r-1},\ldots ,P_k^{r-1}$
\caption{\SM\ Algorithm}\label{alg}
\end{algorithm}

\section{Theoretical Performance of $k$-modes Algorithm in Boolean Block Model}\label{sec:theorem}

In this section, we show that in the Boolean block   model,  $k$-modes algorithm falters, but with the appropriate setting of the rounding function hyperparameter of \SM, we can still obtain good clustering of categorical data.

\begin{theorem}\label{thm:main}
Let $0<p<0.5, q<p$ be constants.
In \BBM\ model, with $k=2$, $|C_1|=|C_2|$, $|D_1|=|D_2|$, $P=\left[\begin{matrix}
p &q\\
q&p
\end{matrix}\right]$  we have the following holds for large $n,d$, and $d$ is sufficiently large in terms of $n$ (say $d=(\log n)^{\Omega(1)}$).  
After $\text{poly}(n)$ iterations, the accuracy score of $k$-modes algorithm is with high probability at most $0.51$. On the other hand, after $\text{poly}(n)$ iterations, the accuracy score of \SM\emph{($t$)}, for every finite $t\ge 1$, is with high probability at least $0.74$. Moreover, if $q$ is sufficiently small and $t=1$, then \SM\emph{($1$)} misclassifies no point with high probability, whereas the accuracy of $k$-modes algorithm remains at most 0.51. 
\end{theorem}
\begin{proof}
We split the proof into two cases based on the initial random seeds. 

Case 1: When the two initial centers are from different blocks.


If we pick a center $c_1$ from the first block, then the expected distance of a point $x$ in the first block to $c_1$ on the first $d/2$ coordinates is $p(1-p)d$ and on the last $d/2$ coordinates is $q(1-q)d$. Thus the distance, $d(x,c_1)$, is concentrated in the interval $p(1-p)d + q(1-q)d \pm O(\sqrt d)$ with high probability (follows from the concentration of sums of Bernoulli variables \cite{Ross97}). Similarly, the typical distance of a point $y$ in the second block from $c_1$ is $(p+q-2pq)d + O(\sqrt d)$ (w.h.p). And the typical distances $d(x, c_2)$ and $d(y, c_2)$ are w.h.p.\ $(p+q-2pq)d \pm O(\sqrt d)$, $p(1-p)d + q(1-q)d \pm O(\sqrt d)$ respectively. So the distance $d(x, c_1)$ is smaller than $d(x, c_2)$, for sufficiently large $d$ (in fact we just need $(p-q)^2 > \frac{1}{\sqrt{d}}$). And the distance $d(y, c_2) < d(y, c_1)$. The high probabilities can be seen to be of the order $1- O(\frac{1}{2^{d^{\Omega(1)}}})$.
Thus for $d \ge (\log n)^{\Omega(1)}$ , by union bound, all the points $x$ in the first block are selected into the cluster with center $c_1$ and the points $y$ in the second block get mapped to cluster with center $c_2$. (we assume $d$ is sufficiently large in terms of $n$ to make sure that the $o(1)$ terms remain small).

Now the $k$-modes algorithm computes the new centers for clusters. The fraction of $1s$ in each coordinate turns out to be less than 0.5 because $q<p<0.5$, thus using majority $\rho^{\text{plu}}$ returns $\vec{0}$ for both the clusters. And we find the new clusters using these centers. These clusters will be random (every vector going randomly into one of the two clusters). The new centers again will be $\vec{0}$ with high probability. And this process repeats where at each step we are essentially obtaining random clusters. So the accuracy is at most $0.51$ (w.h.p.).

If we use uniform rounding $\rho^{\text{uni}}$ instead, the first cluster will pick a center whose first $d/2$ coordinates are Bernoulli random variables with parameter $p +O(\frac{1}{\sqrt d})$, and the last $d/2$ coordinates Bernoulli random variables with parameter $q +O(\frac{1}{\sqrt d})$. Thus we obtain  a vector which resembles a typical vector from the first block.
Similarly, the new second cluster center will resemble a vector in the second block and have first $n/2$ coordinates picked with probability $q +O(\frac{1}{\sqrt d})$ and last $d/2$ with probability $p +O(\frac{1}{\sqrt d})$. More generally the $\SM(t)$ will round to centers with parameters
$ \sim p_t =\frac{p^t}{p^t + (1-p)^t}$ on the first $d/2$ coordinates and $ \sim q_t =\frac{q^t}{p^t + (1-q)^t}$ on the last $d/2$ coordinates. 

Using similar arguments with distances as in the first round, we get that points in the first block are closer to the first center and the points in second block are closer to the second center. For instance, for $\SM(t)$, we get
$d(x, c_1) = (p+p_t -2pp_t)d/2 + (q+q_t -2qq_t)d/2 +O(\sqrt d) < d(x, c_2) = (p+q_t -2pq_t)d/2 + (q+p_t -2qp_t)d/2 +O(\sqrt d)$ and similarly $d(y, c_2) < d(y, c_1)$ w.h.p.
 
Thus we get perfect clusters w.h.p.\ and the process terminates.

Case 2: When the initial centers from the same block.

Assume both the centers are from the first block. Assume after projection to the first $d/2$ coordinates the center $c_1$ is more likely, than center $c_2$, to be close to a random vector where all $d/2$ coordinates are Bernoulli with parameter $q$. Now points of the second block are more likely to be closer to $c_1$.  And let $\alpha$ fraction of the points of the first block be closer to $c_1$ and $1-\alpha$ fraction of the points be close to $c_2$

The first cluster contains $\alpha$ fraction of the first block and a bigger $\beta$ fraction of the second block, and the second cluster will have $1-\alpha$ fraction from the first block  and $1-\beta$ fraction from the second block.

Now the majority rounding can be seen to give $\vec{0}$ as the new centers, so just like in the previous case we get random clusters, then majority again rounds to $\vec{0}$ and the process repeats.

Uniform rounding gives a new center for the first cluster, say $\tilde{c}_1$, where the first $d/2$ coordinates are Bernoulli random variables with parameter $\frac{\alpha p + \beta q }{\alpha+\beta}$ and the last $d/2$ coordinates are Bernoulli random variables with parameter $\frac{\beta p + \alpha q}{\alpha+ \beta} $. The center for the second cluster, say $\tilde{c}_2$, on first $d/2$ coordinates are Bernoulli random variables with parameter $\frac{(1-\alpha)p + (1-\beta)q}{2-\alpha  -\beta} $, and the last $d/2$ coordinates are Bernoulli random variables with parameter  $\frac{(1-\beta)p+(1-\alpha)q}{2-\alpha -\beta} $. With these new centers we can compute the typical distance of a point $x$ in the first block  $d(x, \tilde{c}_1)$ to be: $$\frac{d}{2}\left(\frac{\alpha p + \beta q }{\alpha+\beta} (1-p)+ \left( 1-\frac{\alpha p + \beta q }{\alpha+\beta}\right)p +\frac{\alpha q + \beta p }{\alpha+\beta} (1-q)+ \left( 1-\frac{\alpha q + \beta p }{\alpha+\beta}\right)q\right),$$ 
which simplifies to:
$$
\frac{d}{2}\left(\frac{\alpha p  }{\alpha+\beta} (1-p)+ \left( 1-\frac{\alpha p  }{\alpha+\beta}\right)p +\frac{  \beta p }{\alpha+\beta}+\Gamma(p,q)\right) 
,$$
where $\Gamma(p,q)$ is an expression in terms of $p,q$ such that $\lim_{q\to 0} \Gamma(p,q)=0$.

Next, the typical distance to the second center $d(x, \tilde{c}_2) $ to be: 
\begin{align*}
&\frac{d}{2}\left(\frac{(1-\alpha) p + (1-\beta) q }{2-\alpha-\beta} (1-p)+ \left( 1-\frac{(1-\alpha) p + (1-\beta) q }{2-\alpha-\beta}\right)p \right.\\
&\phantom{sdfljkfh}\left.+\frac{(1-\alpha) q + (1-\beta) p }{2-\alpha-\beta} (1-q)+ \left( 1-\frac{(1-\alpha) q + (1-\beta) p }{2-\alpha-\beta}\right)q\right),
\end{align*} 
with high probability and this simplifies to,
$$\frac{d}{2}\left(\frac{(1-\alpha) p  }{2-\alpha-\beta} (1-p)+ \left( 1-\frac{(1-\alpha) p  }{2-\alpha-\beta}\right)p +\frac{  (1-\beta) p }{2-\alpha-\beta}+\Gamma'(p,q)\right) ,
$$
where $\Gamma'(p,q)$ is an expression in terms of $p,q$ such that $\lim_{q\to 0} \Gamma'(p,q)=0$.

We now use the assumption that $q$ is very small, and first note that $\lim_{q\to 0} \beta=~1$. This is because, as $q$ tends to zero, the distance of any point $y$ in the second block to $\tilde{c}_1$ (resp.\ $\tilde{c}_2$) approaches the sum of the Hamming weight of $\tilde{c}_1$ (resp.\ $\tilde{c}_2$) and  Hamming weight of $y$. Since $\tilde{c}_1$ is lower Hamming weight than $\tilde{c}_2$, then $y$ is closer to $\tilde{c}_1$ than $\tilde{c}_2$. Therefore as $q$ tends to 0, the fraction of points in second block going to the first cluster approaches 1. 

Next, if $\beta$ is very close to $1$ then, $d(x, \tilde{c}_2) $ approaches $2p(1-p)$ and $d(x, \tilde{c}_1) $ approaches $1+p-\frac{2\alpha p^2  }{\alpha+1}  $. By noting that $1+p>2p$ and $\alpha/(1+\alpha)<1$, we have that $d(x, \tilde{c}_1) > d(x, \tilde{c}_2)$ (when $\beta$ is very close to 1). Hence $x$ is mapped to the second cluster w.h.p.. Similarly, we obtain the following inequality for the typical distance of a point $y$ in the second block  to the two centers: $d(y, \tilde{c}_1) <  d(y, \tilde{c}_2) $. Hence points in the second block are mapped to the first cluster. 
\end{proof}

\begin{remark}
Theorem~\ref{thm:main} can be extended to handle multiple clusters in the case when the initial centers are all from different blocks (which would be the case if the seeding is done carefully, such as $D^1$ sampling). However, if this is not the case, we could still obtain similar results for certain parameter regimes, although the arguments would be significantly more involved. Getting the result for most parameter ranges is a much more difficult problem and involves getting probability estimates that are currently out of reach in literature. Also, we don’t explicitly mention concentration bounds in our proof, but most of our arguments use concentration of sums of large numbers of variables. The high probability (w.h.p.) bounds come from such concentration estimates. The dependence of the sample size and the dimension of the data in the probability estimates can also be made more precise and explicit, but might not be very interesting to state formally.
\end{remark}

\section{Experiments}
In Section~\ref{sec:syn}, we construct synthetic datasets in the stochastic block model (that empirically verify Theorem~\ref{thm:main}) and in the so-called corrupted codeword model. Then in Section~\ref{sec:real}, we empirically show that \SM\ performs well on many real-world categorical datasets. Finally, in Section~\ref{sec:analyze}, we empirically analyze the convergence of \SM\ and the sensitivity of its hyperparameter. All the code used to perform the experiments in this section are made available at \cite{code}.

The experiments were performed on Ubuntu 20.04 LTS operating system, Dual Intel(R) Xeon(R) CPU E5-2687W @ 3.10GHz - 32 vCores, with memory 256GB. We use the scikit-learn implementations \cite{scikit-learn} for the clustering algorithms $k$-means, BIRCH, and Gaussian. We use the implementation of clustering algorithm ROCK from the library PyClustering \cite{Novikov2019}.
The hyperparameters, whenever relevant, were determined in a hit-and-trial approach by using a wide range of hyperparameters for each algorithm on every dataset.

\subsection{Synthetic Datasets}\label{sec:syn}

\paragraph{Stochastic Block Model.} 
We considered the stochastic block model and sampled $10^4$ points (each block with $5\times 10^3$ points) in $\{0,1\}^{10^4}$ for various values of $p$ and $q$. We then compared $k$-modes algorithm with \SM$(1)$. Each algorithm ran for 50 epochs with random seeding.  The results are summarized in Figure~\ref{fig:bar}.

\begin{figure*}
    \centering
    \setkeys{Gin}{width=0.69\textwidth}
  \includegraphics[scale=0.5]{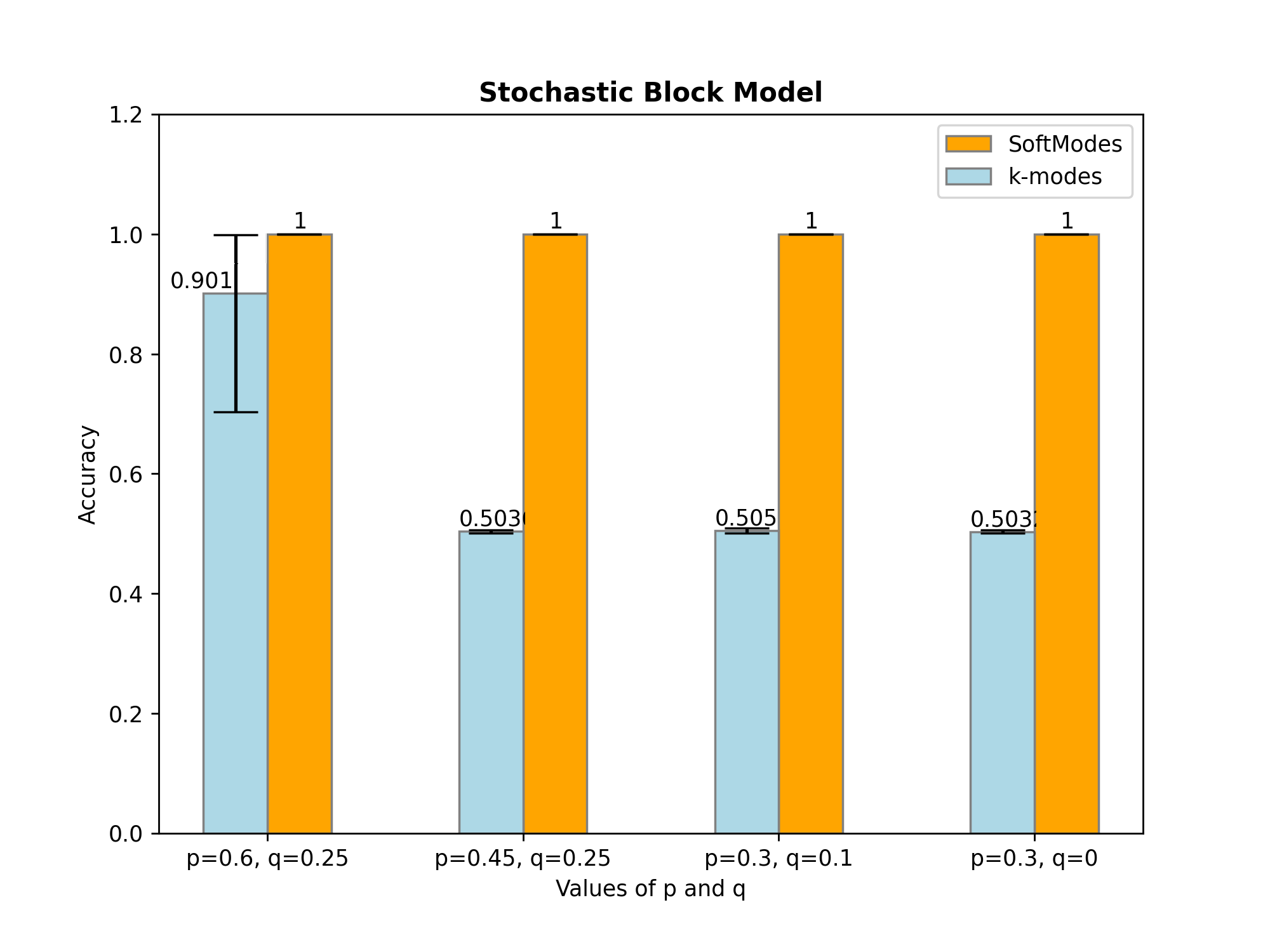} 
  \caption{Bar plots of the experiments on the Stochastic Block Model}
 \label{fig:bar}
    \end{figure*}
    
     From Figure~\ref{fig:bar}, it is clear that when $p<0.5$, \SM$(1)$ classifies perfectly, where as the accuracy of $k$-modes is about 0.5. Moreover, even when $p>0.5$, $k$-modes algorithm even with high accuracy has high variance. These empirical observation are exactly as predicted by Theorem~\ref{thm:main}.
     
\paragraph{Corrupted Codewords Model.} This is a new model that we introduce and is inspired from coding theory, where we are given as parameters the number of samples (denoted by $n$), the number of features (denotes by $d$), the number of clusters (denoted by $k$), and the packing coefficient (denoted by $\varepsilon$). The data points are then sampled as follows:  we first uniformly and independently sample $k$ points (centers) in $\{0,1\}^d$. Then for each center $c$, we obtain each of the $n/k$ points associated with the cluster center $c$ as follows: we flip each coordinate of $c$ uniformly and independently with probability $\varepsilon$. 
Therefore, if $n\ll 2^d$ and $\varepsilon<0.25$, then in this model we obtain $k$ well-separated clusters. This model maybe seen as the Hamming metric counterpart of the mixture of Gaussian model that is extensively studied, particularly in the context of clustering.

We considered the above model and sampled $5\times 10^5$ points in $\{0,1\}^{1000}$ for various values of $k$ with $\varepsilon$ fixed to 0.2. We then compared $k$-modes algorithm and $k$-means algorithm with \SM. Each algorithm ran for 5 epochs with $k$-means++ seeding.  The results are summarized in Figure~\ref{fig:CCMAccuracy}. 

\begin{figure*}
    \centering
    \setkeys{Gin}{width=0.49\textwidth}
\subfloat[Accuracy plot
          \label{fig:CCMAccuracy}]{\resizebox{0.49\textwidth}{!}{\includegraphics[scale=5]{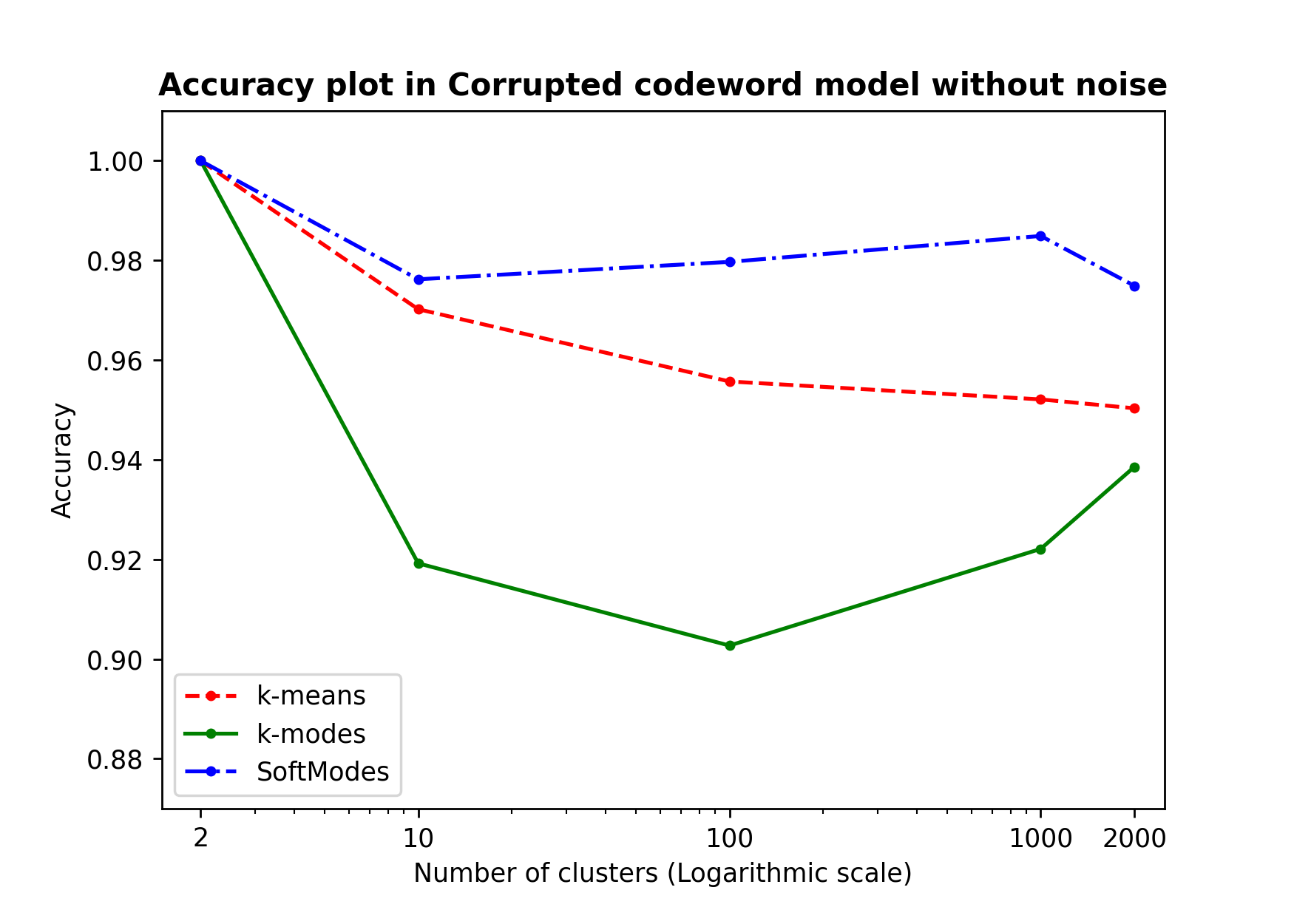}}}
    \hfill
\subfloat[Time plot
          \label{fig:CCMTime}]{\resizebox{0.49\textwidth}{!}{\includegraphics[scale=5]{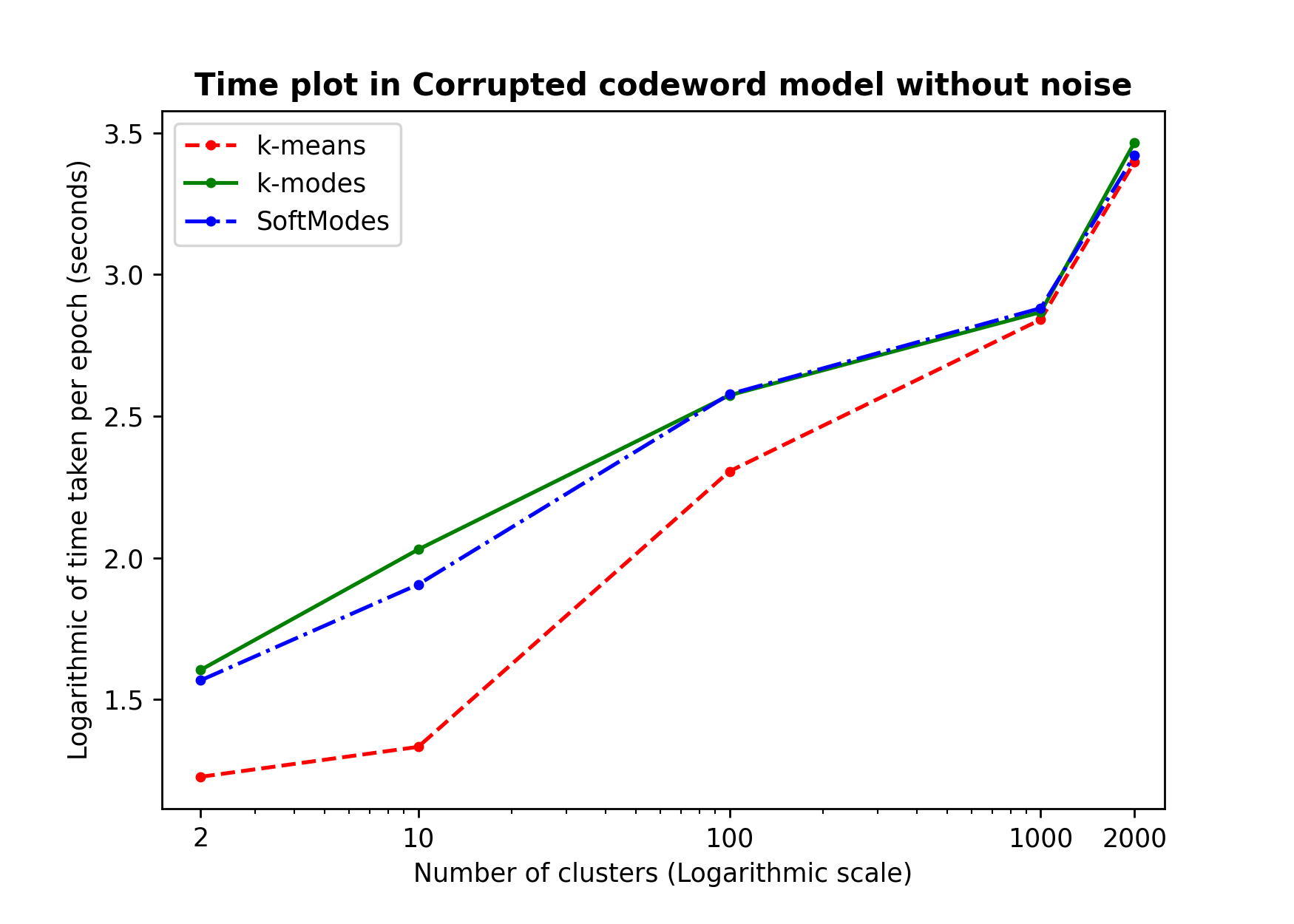}}}
\caption{Line plots of the experiments on the Corrupted codewords model (without noise)}
\label{fig:basicc}
    \end{figure*}

Two important inferences that are clear from Figure~\ref{fig:CCMAccuracy} are that (i) In this model, $k$-means algorithm (which is inherently designed for the Euclidean metric) is surprisingly better than $k$-modes algorithm (which is designed for the 
Hamming metric). This may be attributed to the rounding issues discussed in previous sections. (ii) \SM\ performs better than $k$-means algorithm  as it does not suffer from the rounding drawbacks of $k$-modes algorithm so can better exploit the Hamming metric geometry over $k$-means algorithm.

 In Figure~\ref{fig:CCMTime}, we compare the time needed per epoch for each of the above three algorithms. While $k$-means is noticeably faster  than the other two algorithm for small values of $k$, they all have comparable runtimes for larger $k$. Additionally, we note that rounding function in \SM\ does not make it slower compared to $k$-modes but rather enhances its speed as it converges faster to a local minima.

We note that while our implementation of $k$-modes and \SM\ algorithm involve some parallelization of (distance) computation, they are not optimized, as we only wanted to demonstrate proof of concept. Therefore, it would not be fair to compare directly our implementation of $k$-modes or \SM\ with the scikit-learn implementation of $k$-means algorithm. Thus, we for the time plot we use our own implementation of $k$-means, and note that any optimization that can be or has been done to the $k$-means algorithm can also be done for $k$-modes and \SM. 

Next, we consider the corrupted codewords model but with the additional introduction of a noise parameter $\rho$. Elaborating, we first sample $\rho\cdot n$ points uniformly and independently  from $\{0,1\}^d$ (and label them uniformly at random in $[k]$) and then sample  $(1-\rho)\cdot n$ points in $\{0,1\}^d$ through the corrupted codewords model  (using $k$ random centers and packing coefficient $\varepsilon$). Note that in the  corrupted codewords model with noise parameter $\rho$, the maximum accuracy that any clustering algorithm can achieve is $\frac{\rho}{k} + 1-\rho$. 

We sampled $10^5$ points in $\{0,1\}^{500}$ with noise parameter $\rho$ set to $0.1,0.5,$ and $0.9$, and for various values of $k$ with $\varepsilon$ fixed to 0.2. We then compared $k$-modes algorithm and $k$-means algorithm with \SM. Each algorithm ran for 10 epochs with $k$-means++ seeding.  The results are summarized in Figure~\ref{fig:CCMNoise}.

\begin{figure*}
    \centering
    \setkeys{Gin}{width=0.33\textwidth}
\subfloat[$\rho=0.1$
          \label{fig:rho1}]{\resizebox{0.33\textwidth}{!}{\includegraphics[scale=5]{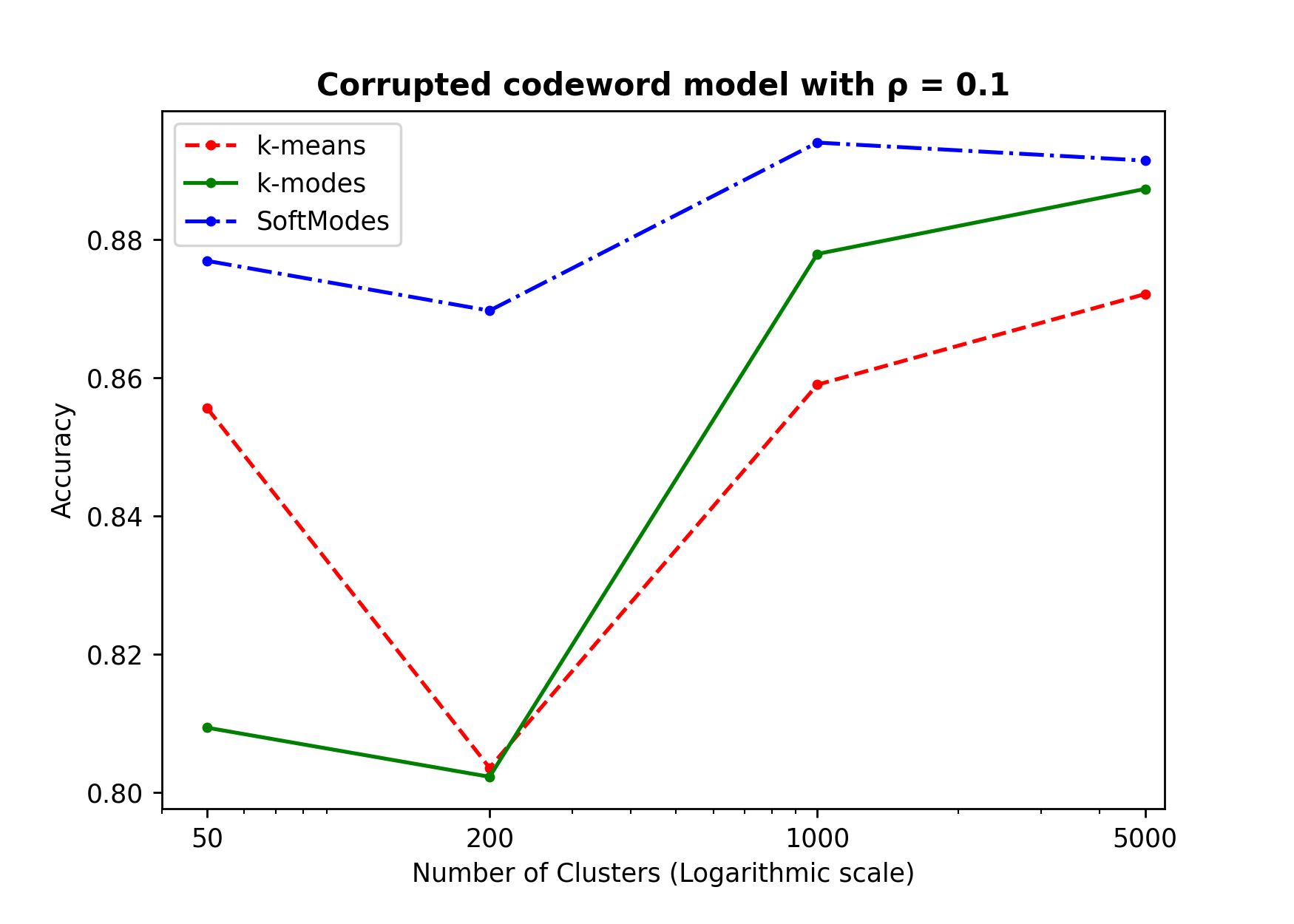}}}
    \hfill
    \subfloat[$\rho=0.5$
          \label{fig:rho5}]{\resizebox{0.33\textwidth}{!}{\includegraphics[scale=5]{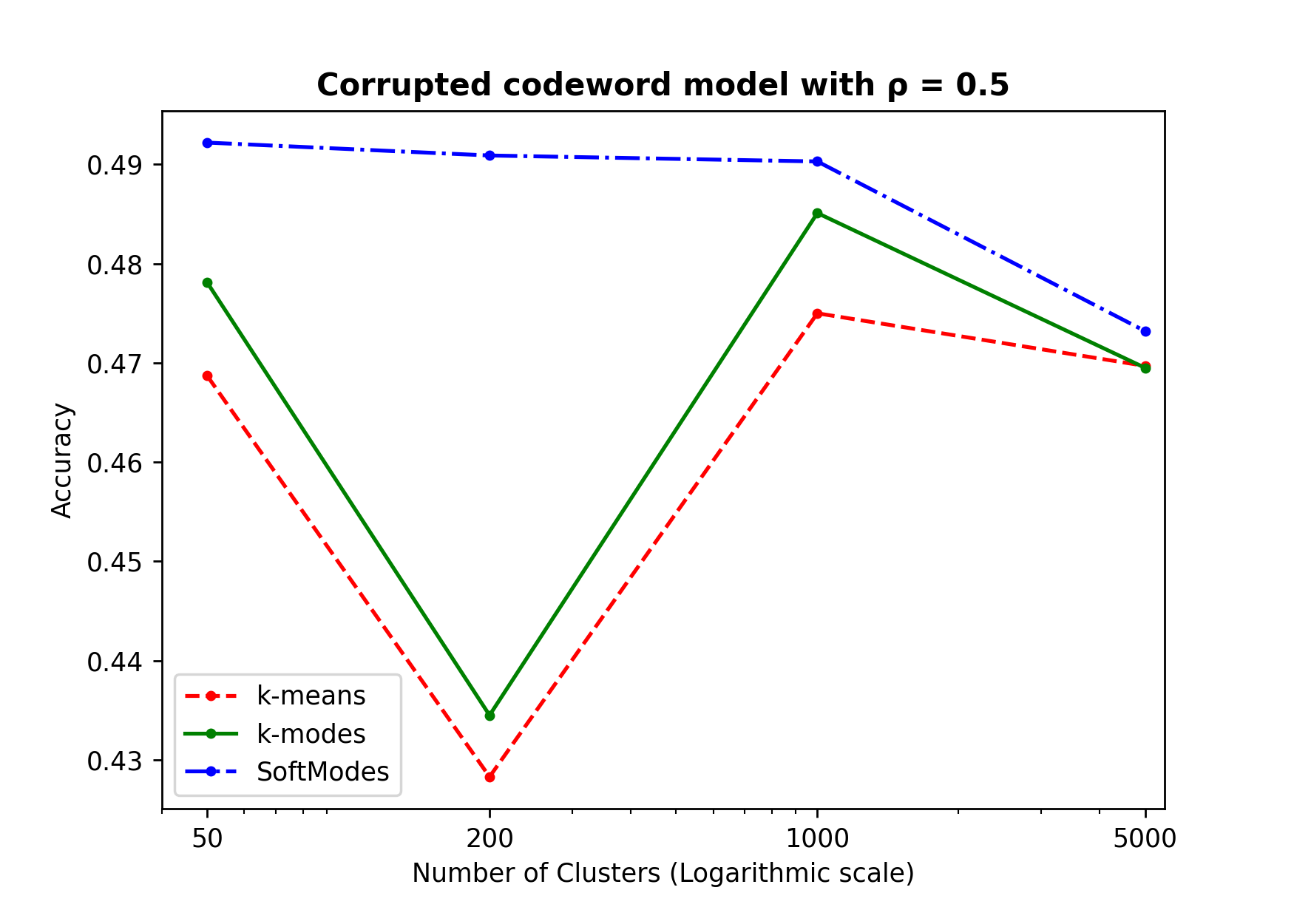}}}
    \hfill
\subfloat[$\rho=0.9$
          \label{fig:rho9}]{\resizebox{0.33\textwidth}{!}{\includegraphics[scale=5]{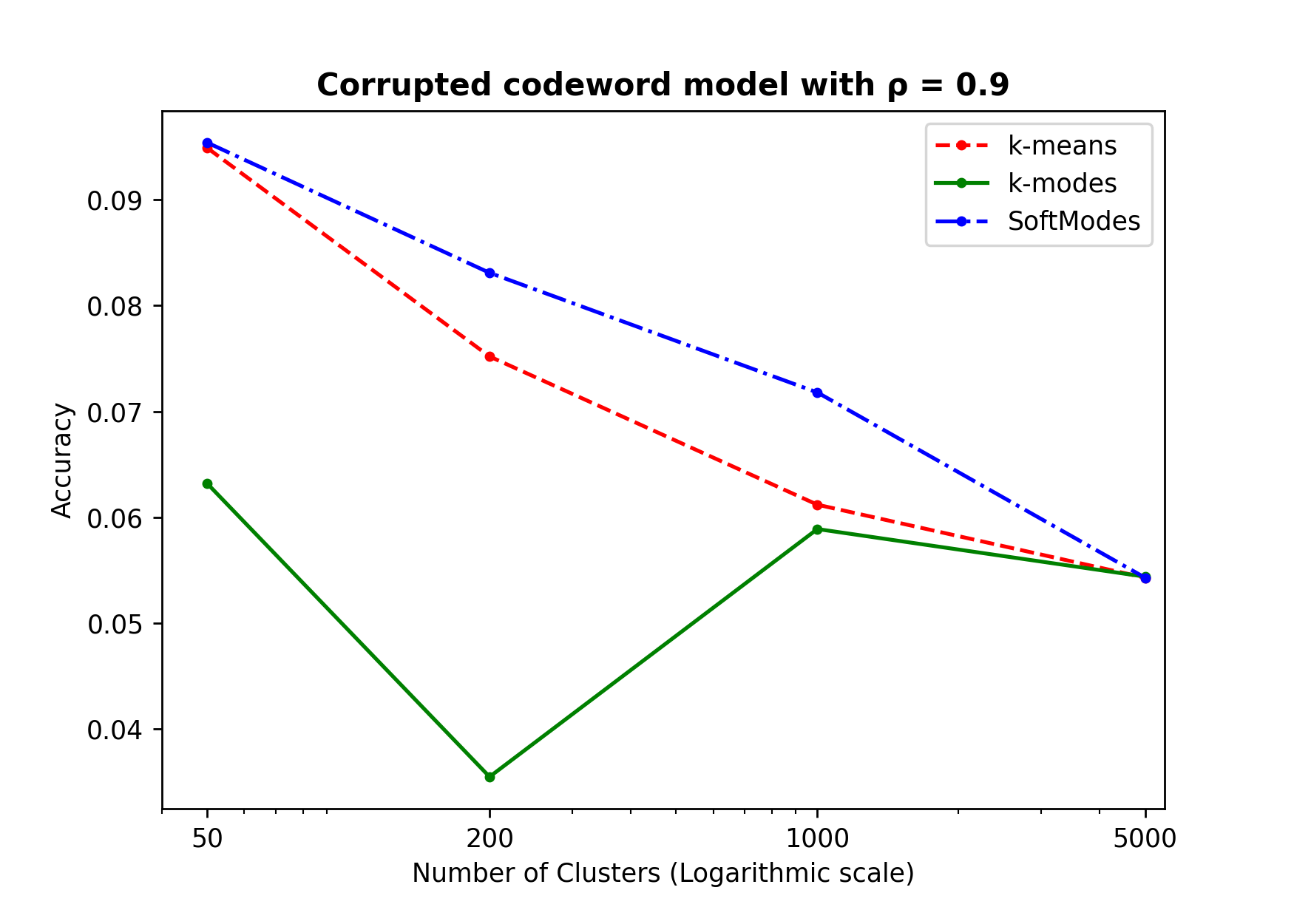}}}
\caption{Line plots of the experiments on the Corrupted codewords model with noise}
\label{fig:CCMNoise}
    \end{figure*}

For all the noise parameter values, \SM\ is consistently having higher accuracy than $k$-means and $k$-modes algorithm. It is also worth noting that as $\rho$ increases, the accuracy even when compared to the theoretical maximum possible of $\frac{\rho}{k}+1-\rho$, goes down for each of the algorithms, as the noise hampers their classification performance even on the "good" samples.  

\subsection{Real World Datasets}\label{sec:real}

We considered five categorical datasets provided by the UCI Machine Learning Repository \cite{Dua:2019}. The information about these datasets is summarized in Table~\ref{infotable}. 

\begin{table}
  \caption{Real World Datasets Information}
  \label{infotable}
  \centering
\resizebox{0.85\linewidth}{!}{  \begin{tabular}{cccccccc}
    \toprule
     Datasets        & Connect-4     &  Nursery &  Mushroom&Tic-Tac-Toe &
      Zoo  \\
    \midrule
    No.\ of Instances & 67557 &12960  &8124 &958 &
    101     \\
    No.\ of Attributes     & 42 &8  &22 &9 &
    17        \\
    No.\ of Classes    & 3 &5  &2 &2 &
    7     \\
    \bottomrule
  \end{tabular}}
\end{table}

We compare the performance of $\SM$\ versus $k$-modes, $k$-means, BIRCH, Gaussian, and ROCK algorithms.  
 We ran $\SM$, $k$-modes, and $k$-means  for 25 epochs with $k$-means++ seeding. All entries in black ran for at most a couple of hours per dataset, per algorithm; the entries in red ran for over a day. ROCK algorithm on Connect-4 dataset did not halt even after a few days.   The results are summarized in Table~\ref{sample-table} (the numbers in blue indicate the hyperparameter values when applicable).

\begin{table}
  \caption{Performance on Real World Categorical Datasets}
  \label{sample-table}
  \centering
  \begin{NiceTabular}{ccccccc}
    \toprule
     Datasets        & $k$-modes    & \SM & $k$-means & BIRCH & Gaussian & ROCK    \\
    \midrule
    
\Block{2-1}{Connect-4} & 0.3836 & {\bf 0.4355}  & 0.3747 & 0.4051 & 0.4191 & -      \\
       & {\scriptsize$\pm$0.0176} &{\scriptsize$\pm$0.0359$({\color{blue}1.5})$}  &{\scriptsize$\pm$0.0035} &  & {\scriptsize$\pm$0.0196} &        \\
       
\Block{2-1}{Nursery} & 0.3038 &  {0.3192}  & 0.2767 & 0.2861 & { 0.323} & {\color{red}\bf 0.3332}     \\
       & {\scriptsize$\pm$0.0184} &{\scriptsize$\pm$0.021$({\color{blue}3})$}  &{\scriptsize$\pm$0.0205} &  &{\scriptsize$\pm$0.0595} &  {\scriptsize$({\color{blue}3})$}     \\

       \Block{2-1}{Mushroom} & 0.7753 & 0.8837  & {\bf 0.8902} & {\bf 0.8902} &0.7987 & {\color{red}0.5201}   \\
       & {\scriptsize$\pm$0.119} &{\scriptsize$\pm$0.0073$({\color{blue}3})$}  &{\scriptsize$\pm$0 } & &{\scriptsize$\pm$0.117}  &   {\scriptsize$({\color{blue}5})$}        \\    

   \Block{2-1}{Tic-Tac-Toe} & 0.5538 & {\bf 0.5817}  & { 0.5678} & 0.5135 & { 0.5678} & 0.5803  \\
       & {\scriptsize$\pm$0.0441} &{\scriptsize$\pm$0.0224$({\color{blue}3.5})$}  &{\scriptsize$\pm$0.0208} &  &{\scriptsize$\pm$0.0208}  &     {\scriptsize$({\color{blue}2})$}   \\


       
       \Block{2-1}{Zoo} &0.7707 & {\bf 0.7986}  & 0.7346 & 0.7326& 0.7801 &0.792      \\
       & {\scriptsize$\pm$0.0779} &{\scriptsize$\pm$0.0774$({\color{blue}3})$}  &{\scriptsize$\pm$0.0747} &  &{\scriptsize$\pm$0.0797}  & {\scriptsize$({\color{blue}2})$}       \\

       \bottomrule
  \end{NiceTabular}
\end{table}

On every dataset \SM\ (with appropriate setting of the hyperparameter) has noticeably higher accuracy than $k$-modes.  Indeed, some of these datasets are better suited to be considered in the Euclidean metric, and therefore in some cases, $k$-means or Gaussians outperform k-modes and \SM. This still does not take away from the message of the paper that \SM\ is a strictly better alternative to $k$-modes. Moreover, even for datasets in which \SM\ is not giving the highest accuracy, it is quite close in accuracy to the best algorithm. 

We note that in terms of asymptotics, the runtime of \SM, $k$-modes, $k$-means, and Gaussian is $O(ndk \cdot \max_{\text{iter}})$, where $n$ is the number of instances, $d$ is the number of attributes, $k$ is the number of classes, and $\max_{\text{iter}}$ is the hyperparameter that indicates the maximum number of iterations the algorithm could run. This runtime is significantly less than the runtime of BIRCH and ROCK, which are hierarchical clustering-based methods that takes a $O(n^2 d)$ time. 

\subsection{Empirical Analysis}\label{sec:analyze}

Finally, in order to analyze the sensitivity of the hyperparameter and rate of convergence, we considered the corrupted codeword model (without noise) and sampled $10^5$ points in $\{0,1\}^{200}$ with $k$ fixed to 100 and  $\varepsilon$ fixed to 0.3. We then ran \SM, $k$-modes, and $k$-means algorithm for 10 epochs with random seeding and noted down the results in Figure~\ref{fig:common}.

\begin{figure*}
    \centering
    \setkeys{Gin}{width=0.49\textwidth}
\subfloat[ 
          \label{fig:subfig-a}]{ \includegraphics[scale=5]{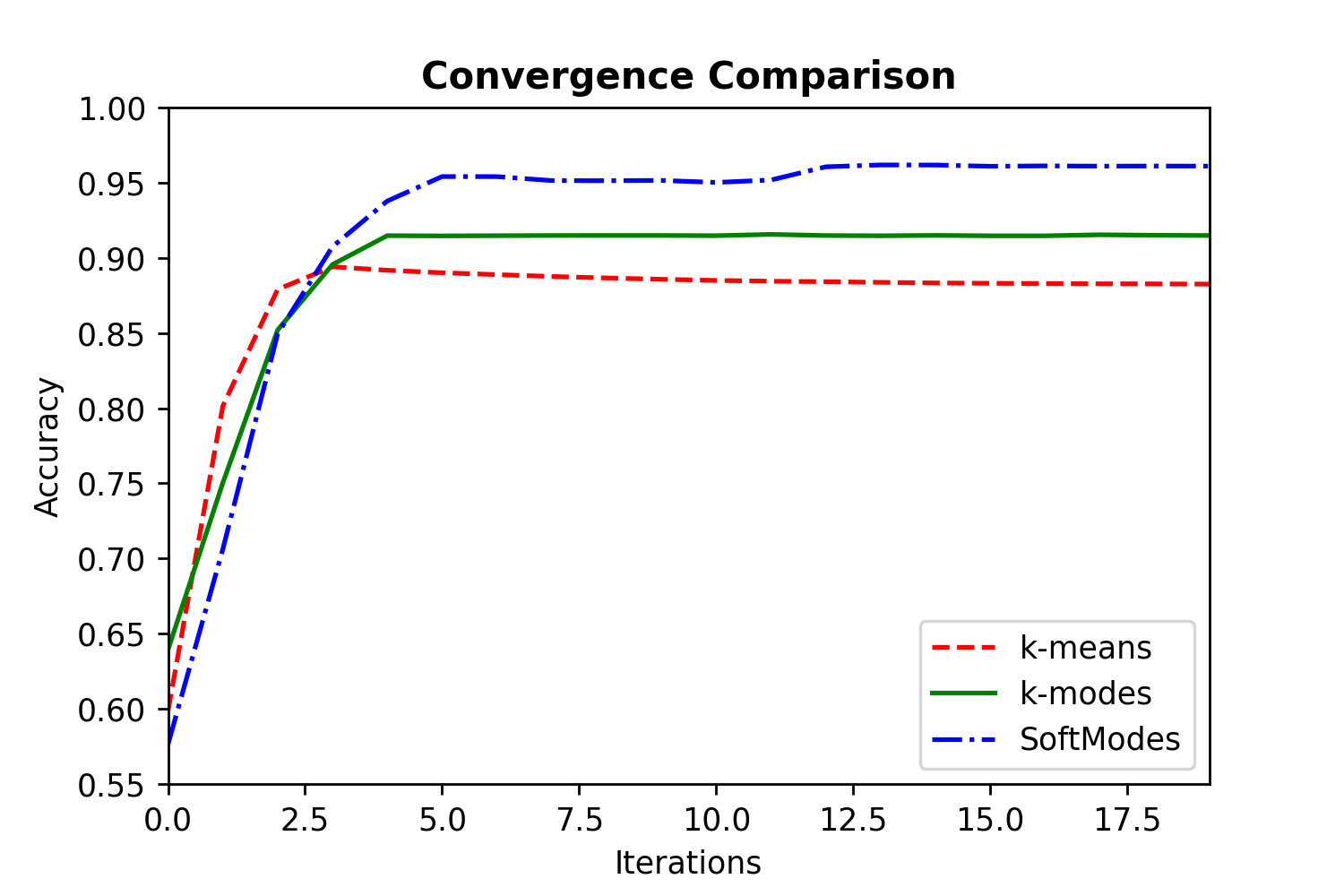}}
    \hfill
\subfloat[ 
          \label{fig:subfig-b}]{ \includegraphics[scale=5]{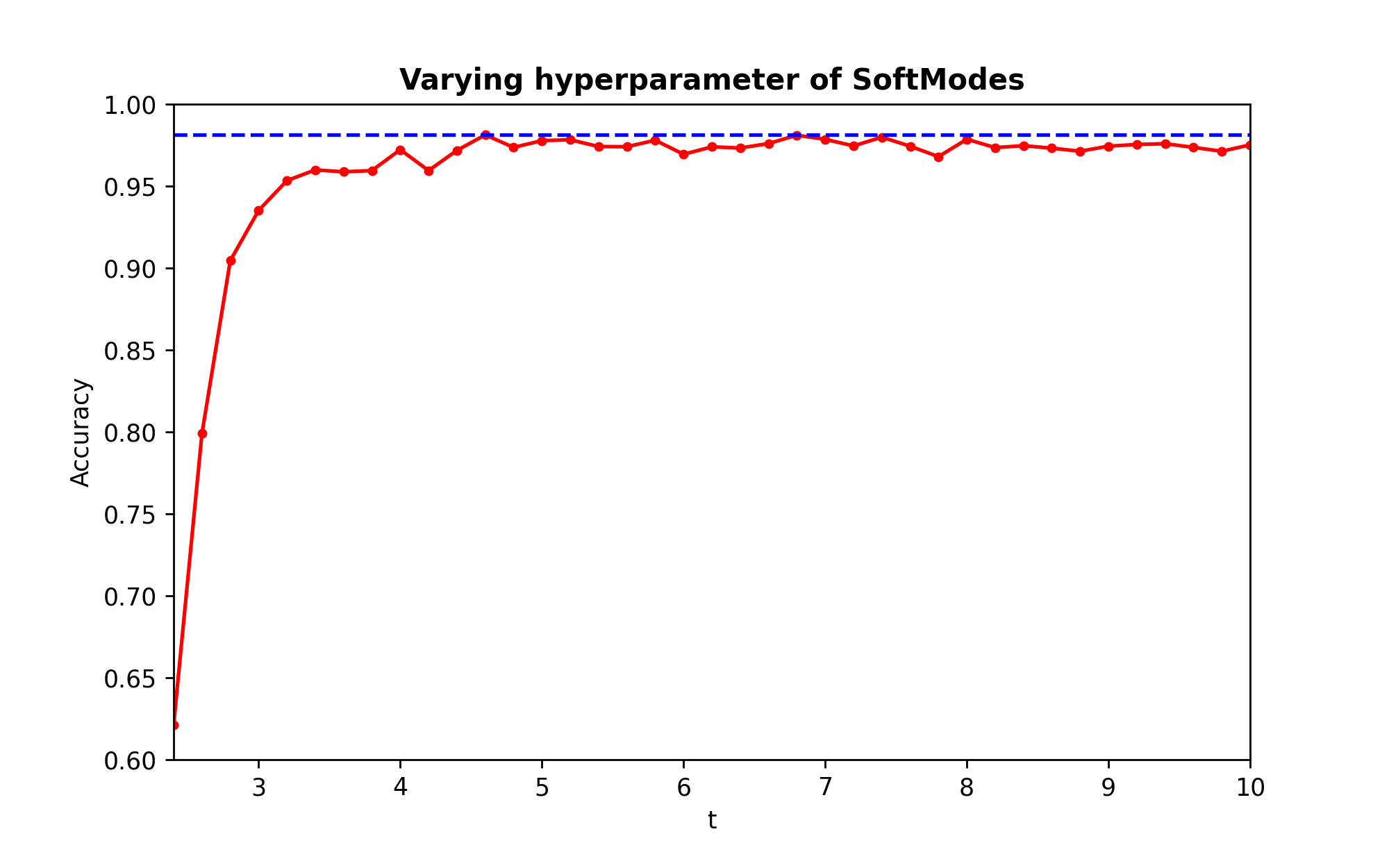}}
    \caption{Empirical Analysis of Convergence and Hyperparameter sensitivity of \SM}
\label{fig:common}
    \end{figure*}
    
While $k$-means converges slightly faster to its local minima than \SM\ in Figure~\ref{fig:subfig-a}, the latter continues to make slow but steady progress. Also it is worth noting that rate of convergence of $k$-modes is less than \SM .    Finally, we remark that the hyperparameter is not very sensitive after a particular threshold is crossed (as seen in Figure~\ref{fig:subfig-b}) and this is expected as the performance of \SM\ is lower bounded  by the performance of $k$-modes algorithm.

    As a side note, we remark that \SM\ is sensitive to initialization and an initial seeding algorithm, like k-means++ seeding, does give observably better accuracy results than random seeding. Note that over the Hamming metric, the $D^1$ and $D^2$ sampling have identical distributions.

\section{Conclusion}
In this paper, we identified (theoretically and empirically) the poor performance of the $k$-modes algorithm in the Boolean Block   model. We remedied the situation by suggesting a soft rounding variant of $k$-modes, namely, \SM. We remark that \SM, much like $k$-means, and $k$-modes can be  parallelized to yield scalability. We leave it as an open question if one come up with a more sophisticated rounding family of functions for which one can theoretically prove constant factor approximation of the $k$-modes objective. Another open question is to theoretically prove that \SM\ performs better than $k$-means in the corrupted codewords model. 
Finally, 
it would be interesting to prove performance guarantees for \SM  when $p > 0.5$, or remove the assumption on $q$ to be sufficiently small for the last conclusion in Theorem~\ref{thm:main} to hold.

\subsection*{Acknowledgment}
This work was supported by a
grant from the Simons Foundation, Grant Number 825876, Awardee Thu D. Nguyen.

\bibliographystyle{alpha}
\bibliography{refs}


\end{document}